\documentclass[12pt]{article}
\usepackage{arxiv}

\usepackage[utf8]{inputenc} 
\usepackage[T1]{fontenc}    
\usepackage{hyperref}       
\usepackage{url}            
\usepackage{booktabs}       
\usepackage{amsfonts}       
\usepackage{nicefrac}       
\usepackage{microtype}      
\usepackage{lipsum}		
\usepackage{graphicx}
\usepackage{doi}
\usepackage{subfig}

\usepackage{custom}
\usepackage{amsmath}
\usepackage{amssymb}
\usepackage{amsthm}
\usepackage{exsheets}
\usepackage[shortlabels]{enumitem}
\usepackage{fancyhdr}
\usepackage{comment}

\newtheorem{theorem}{Theorem}
\newtheorem{lemma}{Lemma}
\newtheorem{assumption}{Assumption}

\makeatletter
\def\set@curr@file#1{\def\@curr@file{#1}} 
\makeatother

\title{Optimistic Policy Iteration for MDPs with Acyclic Transient State Structure}
\author{%
 Joseph Lubars\\
 University of Illinois at Urbana-Champaign\\
 \texttt{lubars2@illinois.edu}\\
 \And
 Anna Winnicki\\
 University of Illinois at Urbana-Champaign\\
 \texttt{annaw5@illinois.edu}\\
 \And
 Michael Livesay\\
 Sandia National Laboratories\\
 \texttt{mlivesa2@illinois.edu}
 \And
 R. Srikant\\
 University of Illinois at Urbana-Champaign\\
 \texttt{rsrikant@illinois.edu}\\
}
\begin{document}
\maketitle
\begin{abstract}%
 We consider Markov Decision Processes (MDPs) in which every stationary policy induces the same graph structure for the underlying Markov chain and further, the graph has the following property: if we replace each recurrent class by a node, then the resulting graph is acyclic. For such MDPs, we prove the convergence of the stochastic dynamics associated with a version of optimistic policy iteration (OPI), suggested in \cite{tsitsiklis2002convergence}, in which the values associated with all the nodes visited during each iteration of the OPI are updated. %
\end{abstract}

\section{Introduction}

Policy iteration is a key computational tool used in the study of Markov Decision Processes (MDPs) and Reinforcement Learning (RL) problems. In traditional policy iteration for MDPs, at each iteration, the value function associated with a policy is computed exactly and a new policy is chosen greedily with respect to this value function \cite{bertsekasvolI, bersekasvolII, bertsekastsitsiklis, suttonbarto}. It can be shown that using policy iteration, the value function decreases with each iteration. In the case of a finite state and action space, the optimal policy is reached in a finite number of iterations. However, computing the exact value function corresponding to each policy can be computationally prohibitive or impossible, especially in an RL setting where the MDP is unknown.

To analyze these settings, optimistic policy iteration (OPI) methods have been studied which assume that at each iteration, only a noisy estimate of the exact value function for the current policy is available. We consider the variant studied in  \cite{tsitsiklis2002convergence}, where at each iteration, we only have access to a noisy, but unbiased, estimate of the value function associated with a policy. This estimate is obtained by simulation using a Monte Carlo approach. The Markov process corresponding to a particular policy is simulated and the corresponding value function is estimated by taking the infinite sum of discounted costs. The key idea in \cite{tsitsiklis2002convergence} is to use stochastic approximation to update the value function using the noisy estimates. Their main results consider a synchronous version of OPI where the value functions of all states are updated simultaneously, but extensions to cases where an initial state is chosen randomly are discussed.

In this variant of OPI, we have a choice of updating the value associated with the initial state selected at each iteration or the values of all states visited in the Monte Carlo simulation at each iteration. In the former case, the results in \cite{tsitsiklis2002convergence} apply almost directly. In this paper, we provide a convergence proof for the latter case under some structural assumptions about the MDP. We also extend the results to the following cases:  (i) stochastic shortest-path problems (see \cite{Yuanlong} for an extension of the work in \cite{tsitsiklis2002convergence} to stochastic shortest-path problems), (ii) zero-sum games (see \cite{patekthesis} for extensions of MDP tools to zero-sum games), and (iii) aggregation, when we know apriori which states have the same value functions.

\subsection{Main Contributions and Related Work}

A common application of reinforcement learning is to games such as chess and Go. In such games, the same state cannot be reached a second time. For example, in chess, due to the rule that the game is considered to be a draw if the same state is reach thrice, each state of the chessboard is augmented to include the number of visits to the state. With this augmentation, the augmented state can only ever be visited once. Motivated by these considerations and the fact that some form of exploration is used to ensure that all states are visited during training, in this paper, we assume that every stationary policy leads to the same Markov chain graph with the following property: the transient states have an acyclic structure. Specifically, we allow recurrent classes in our model of the MDP. For such a model, we establish that the version of Optimistic Policy Iteration in \cite{tsitsiklis2002convergence} converges if the value of every visited state is updated in each iteration. 

We note that the term OPI is not standard; for example, OPI refers to a different algorithm in \cite{bertsekasvolI,bersekasvolII}. Additionally, in \cite[Section 5.4]{bertsekastsitsiklis}, the algorithm considered in this paper is referred to as the asynchronous optimistic TD(1) algorithm. However, we have chosen to call it OPI as in the paper by \cite{tsitsiklis2002convergence}.  We also note that there are a large number of reinforcement learning algorithms whose convergence has been studied and established; see \cite{bertsekas2019reinforcement}. However, the algorithm studied in \cite{tsitsiklis2002convergence} is somewhat unique: at each iteration, one follows the entire trajectory of a greedy policy from each state to estimate the value of the policy and uses the estimate of the cost of the trajectory (from each state) to update the value function. To the best of our knowledge, the convergence of the asynchronous version of such an updating scheme has not been studied in the literature and is mentioned as an open problem in \cite{tsitsiklis2002convergence}. A similar update is also used in the famous AlphaZero algorithm \cite{silver2017mastering} where a game is played to conclusion and the values of all the states encountered are updated based on the outcome of the game. We note, however, that AlphaZero has many other features which are not studied here.

We first present our results for standard MDPs. Since our structural assumption on the MDP is motivated by games, we extend our results to zero-sum games later.  Additionally, since most RL algorithms for large MDPs use some form of function approximation to estimate the value function, we also extend our results to a very special case of function approximation, namely, state aggregation. When we consider state aggregation, we assume that all states in a cluster belong to the same level (i.e., same depth from the root). This assumption is similar to function approximation algorithms for finite-horizon MDPs where a separate function is used for each time step; see \cite{Jordan, jin2020provably}.

\section{Definitions and Assumptions}

Let $X$ be a discounted Markov Decision Process (MDP) with discount factor $\alpha \in (0, 1)$ and finite state space $S = \{1, \ldots, n\}$. Denote the finite action space associated with state $i \in S$ by $\scriptA(i)$. When action $u \in \scriptA(i)$ is taken at state $i$, we let $P_{ij}(u)$ be the probability of transitioning from state $i$ to state $j$. For every state and action pair, $(i, u),$ we are also given a finite, deterministic cost $c(i, u)$, $c \geq 0$, of being in state $i$ and taking action $u$.

A policy $\mu$ is a mapping $\mu: S \to \cup_{i\in S} \scriptA(i)$. Policy $\mu$ induces a Markov chain $X^\mu$ on $S$ with transition probabilities
\begin{align*}
\P(X_{k+1}^{\mu} = j | X_k^{\mu} = i) = P_{ij}(\mu(i)) \quad \forall i, j \in S, 
\end{align*} where $X_k^\mu$ is the state of the Markov chain after $k \in \mathbb{N}$ time steps. %

We assume that the distribution for the initial state $X_0^{\mu}$ is $p$ for all policies $\mu$. The distribution $p$ and $P_{ij}(\mu(i)) \text{ } \forall i,j \in S$ determine $q_\mu(i)$, the probability of Markov chain $X^{\mu}$ ever reaching state $i$. In other words,
\begin{equation}
    P(X_k^{\mu} = i \text{ for some } k, 0 \leq k < \infty) = q_\mu(i). \label{q_mu_def}
\end{equation}
In order to ensure sufficient exploration of all of the states, we assume the following:
\begin{assumption}\label{assumption_reachability}
$q_\mu(i) > 0 \;\forall \mu, i.$
\end{assumption} 

Since there are finitely many policies, there exists $\delta$ such that $q_\mu \geq \delta > 0.$ Furthermore, we make the following assumption about state transitions in our MDP:
\begin{assumption}\label{assumption_common_transitions}
    For any states $i, j \in S$ and actions $u, v \in \scriptA(i)$, $P_{ij}(u) > 0$ if and only if $P_{ij}(v) > 0$. 
\end{assumption}
Thus, the set of states that can be reached from any state in one step is the same under any policy. The above assumptions are usually satisfied in practice since one explores all actions with at least some small probability in each state; examples of such exploration strategies include epsilon-greedy and Boltzmann explorations. Given this assumption, we can define a one-step reachability graph of our MDP independently of any policy. We define the reachability graph as the directed graph $G = (S, E)$ where $S = \{1, \ldots, n\}$ and $E = \{(i,j): P_{ij}(\mu(i))>0 \text{ for some } \mu\}$. 

We now further classify $S$ into transient and recurrent classes as follows: 
\begin{equation*}
    S = \scriptT \sqcup \scriptR_1 \sqcup \scriptR_2 \sqcup \ldots \sqcup \scriptR_m 
\end{equation*}
Here, $\scriptT = 1, \ldots, L$ where $L < n$ is the set of transient states and $\scriptR_1, \scriptR_2, \ldots, \scriptR_m$ are disjoint, irreducible, closed recurrent classes.  Assumption \ref{assumption_common_transitions} allows us to drop the dependence on policy $\mu$ in the decomposition.

We are now ready to state our third assumption, which is also illustrated in Figure \ref{fig:assumption_acyclic}.

\begin{assumption}\label{assumption_acyclic_transient}
The subgraph of the reachability graph induced by the set of transient states $\scriptT$ which we denote by $G(\scriptT)$ is acyclic.
\end{assumption}

\begin{figure}
    \centering
    \includegraphics[width=0.5\textwidth]{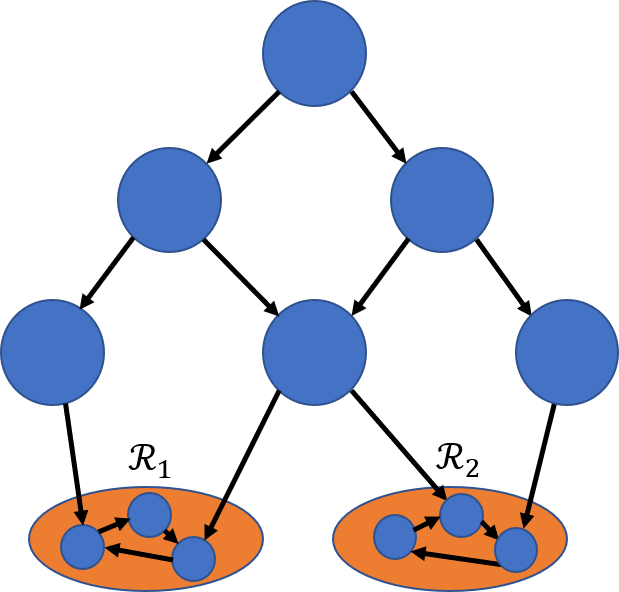}
    \caption{An illustration of Assumption \ref{assumption_acyclic_transient}. The reachability graph contains several recurrent classes (here, in orange), and the remainder of the reachability graph is acyclic.}
    \label{fig:assumption_acyclic}
\end{figure}

Although restrictive, this assumption naturally arises in some problems. For example, many existing works, such as \cite{Jordan}, assume a finite time horizon. They augment the state with a time-dependent parameter, naturally making the state transitions acyclic, as it is impossible to transition to a state-time pair with the time being in the past.

\section{Reinforcement Learning Preliminaries}\label{RLPre}
To define and analyze our algorithm, we will need several standard definitions and results from dynamic programming and reinforcement learning. First, we define the cost-to-go or value function $J^{\mu}(i)$ $\forall i \in S$ as the expected  cumulative discounted cost when following policy $\mu$, starting from state $i$:
\begin{align*}
J^{\mu} (i)= \E \Big [ \sum_{k = 0}^\infty \alpha^k c(X_k^{\mu}, \mu(X_k^{\mu})) | X_0^{\mu} = i \Big ]. 
\end{align*}
It can be shown that $J^{\mu}$  solves the Bellman equation:
\begin{equation}
	J^{\mu}(i) = c(i, \mu(i)) + \alpha \sum_{j=1}^n P_{ij}(\mu(i)) J^\mu(j). \label{eq:bellmaneq}
\end{equation}
Now, we define an optimal policy, $\mu^*$, to be a policy that solves $J^*:=\min_\mu J^\mu$. Under our assumptions, $\mu^*$ always exists. $J^*$ is known as the optimal value function and satisfies the following Bellman equation:

\begin{equation}
	J^*(i) = \min_{u \in \scriptA(i)} \Bigg[ c(i, u) +  \alpha \sum_{j=1}^n P_{ij}(u) J^*(j) \Bigg ]. \label{Bell}
\end{equation}

For an arbitrary vector, we introduce the optimal Bellman operator:
\begin{equation}
   (T J)(i) = \underset{u \in \scriptA(i)}\min \Bigg[ c(i, u) + \alpha \sum_{j=1}^n P_{ij}(u) J(j)\Bigg] \label{eq:Topt}.
\end{equation}

Our primary goal is to find $J^*$ and $\mu^*$. Towards the objective, we introduce the Bellman operator $T_\mu: \mathbb{R}^n \to \mathbb{R}^n$ where for $J \in \mathbb{R}^n,$ the $i$th component of $T_{\mu}J$ is
\begin{align}
(T_\mu J)(i) = c(i, \mu(i)) + \alpha \sum_{j=1}^n P_{ij}(\mu(i)) J(j), \label{bellmanop} 
\end{align} so that \eqref{eq:bellmaneq} can be written as $J^\mu = T_{\mu} J^{\mu}$.

Policy iteration is a basic iterative algorithm for finding $J^*$ and $\mu^*$. Each iteration starts with an estimate of the value function $J_t$ and then performs ``policy improvement'' to produce a policy $\mu_t$ and ``policy evaluation'' to produce the next estimate of the value function $J_{t+1}$. Policy improvement finds the greedy policy with respect to $J_t$ by solving $\mu_{t} = \arg\min_{\mu} T_{\mu} J_t$. Policy evaluation finds the value function $J^{\mu_t}$ of the current policy by solving the Bellman equation \eqref{eq:bellmaneq}, and sets $J_{t+1} = J^{\mu_t}$. The key to convergence is that $J_t$ strictly improves at every step, in the sense that $J_{t+1} \leq J_t$, with equality if and only if $\mu_t = \mu^*$ and $J_t = J^*$. Since $\mu_t$ belongs to a finite set, policy iteration is guaranteed to converge in a finite number of iterations.

Calculating $J^{\mu_t}$ in each step of policy iteration can be computationally expensive and the results of policy iteration cannot be easily extended when the probabilities of transitioning between states and costs are not known, so \emph{optimistic policy iteration} refers to a variant of policy iteration where some approximation of $J^{\mu_t}$ is used instead of calculating $J^{\mu_t}$ directly. In \cite{tsitsiklis2002convergence}, assuming that $p_{ij}(u)$ are known for all $i, j \in S$ and $u \in \scriptA(i)$ and that $c(i, u)$ are known for all $i \in S$ and $u \in \scriptA(i)$, it was shown that an optimistic policy iteration algorithm using Monte Carlo simulations for policy evaluation converges to $J^*$. Here, we consider a variant suggested in \cite{tsitsiklis2002convergence} which can lead to faster convergence. 

\section{The Algorithm}\label{TheAlgorithm}

The algorithm we consider is as follows. Like policy iteration, we start with an initial vector $J_0=0$ and iteratively update $J_t$ for all $t$. For each update at time $t$, we take vector $J_t$ and obtain 
\begin{align}
\mu_t = \arg\min_{\mu}(T_\mu J_t), \label{greedypol}
\end{align} which is the greedy policy with respect to $J_t$. Then, the algorithm independently selects a state according to nonuniform probabilities $p(i), i \in S$. We then simulate a trajectory that starts at state $i$ and follows policy $\mu_t$ at time $t$. The trajectory is a realization of a Markov chain $X_k^{\mu_t}$ where $k \in \mathbb{N}$ and $X_0^{\mu_t} = i$.

Instead of using \eqref{eq:bellmaneq} to compute $J^{\mu_t}$, we use this trajectory to generate an unbiased estimate $\tilde{J}^{\mu_t}$ of $J^{\mu_t}$ using the tail costs of the first time each state is visited by the trajectory. 

To formalize $\tilde{J}^{\mu_t}(i)$, we introduce the hitting time $N_t(i)$ of state $i$ in the trajectory $X_k^{\mu_t}$ as follows:
\begin{align*}
N_t(i) := \inf\{n: X_n^{\mu_t} = i\}.
\end{align*}
When $N_t(i)$ is finite, $\tilde{J}^{\mu_t}(i)$ can be defined in terms of $N_t(i)$ as
\begin{align*}
    \tilde{J}^{\mu_t}(i) := \sum_{k=N_t(i)}^\infty \alpha^{k-N_t(i)} c(X_k^{\mu_t}, \mu_t(X_k^{\mu_t})).
\end{align*} Otherwise, $\tilde{J}^{\mu_t}(i) = 0$. Then, for every state visited by the trajectory, $X_k^{\mu_t}$, we update $J_{t+1}$ as follows:
\begin{equation}
J_{t+1}(i) =  
	      \begin{cases}
               (1-\gamma_t(i))J_t(i) + \gamma_t(i)\tilde{J}^{\mu_t}(i) & \text{ if $i \in X_k^{\mu_t} $}\\
               J_t(i) & \text{ if $i \notin X_k^{\mu_t}$},
            \end{cases} \label{eq:update0}
\end{equation}
where $\gamma_t(i)$ is a component-dependent step size. Recall that $J_0$ is a deterministic vector. In order to analyze this algorithm, it is helpful to rewrite it in a form similar to a stochastic approximation iteration. We introduce a random variable $w_t$ to capture the noise present in $\tilde{J}^{\mu_t}(i)$. When $i \notin X^{\mu_t}_k$, we define $w_t(i)=0$. Otherwise, we let $w_t = \tilde{J}^{\mu_t}(i) - J^{\mu_t}(i)$. With this choice, we can rewrite our iterates as
\begin{equation}
J_{t+1}(i) =  
	      \begin{cases}
               (1-\gamma_t(i))J_t(i) + \gamma_t(i)(J^{\mu_t}(i) +  w_t(i))& \text{  if $i \in X_k^{\mu_t}$},\\
               J_t(i) & \text{ otherwise}.\\
            \end{cases} \label{eq:updatewithoutv}
\end{equation}

We now introduce a random variable $v_t$ which incorporates the randomness present in the event $i \in X_k^{\mu_t}$, similar to the random variable $v_t$ used in \cite{tsitsiklis2002convergence}, and rewrite \eqref{eq:updatewithoutv} as  
\begin{equation}
J_{t+1}(i) = (1-q_{\mu_t}(i)\gamma_t(i))J_t(i) + q_{\mu_t}(i)\gamma_t(i)(J^{\mu_t}(i) +  w_t(i) + v_t(i)) \label{eq:updatewithv}
\end{equation} where 
\begin{align*}
v_t(i) = \frac{1}{q_{\mu_t}(i)} (\I_{i \in X_k^{\mu_t}} - q_{\mu_t}(i) ) (J^{\mu_t}(i) +  w_t(i) - J_t(i) ).
\end{align*} 
Recall that $q_{\mu_t}(i)$ is the probability of ever reaching node $i$ using policy $\mu_t$.

\section{Main Result}
The main result of our paper is establishing the convergence of the above algorithm. However, in order to establish convergence, we have to specify the step size $\gamma_t(i).$ We consider two choices of step sizes: deterministic, state-independent step sizes and state-dependent step sizes which decrease when state $i$ is visited. These step sizes are assumed to satisfy fairly standard assumptions for stochastic approximation algorithms. We assume there is some deterministic function $\beta: \mathbb{N} \to \R^+$ such that
\begin{equation*}
    \sum_{t=0}^\infty \beta(t) = \infty, \qquad \sum_{t=0}^\infty \beta^2(t) < \infty,
\end{equation*}
and we assume that there exists some constant $T$ such that $\beta(t)$ is nonincreasing for $t > T$. Then, our choices of step sizes are:
\begin{itemize}
    \item Deterministic step size $\gamma_t(i) = \beta(t)$: This choice is simple to implement and does not depend on state $i$, but may converge slower than necessary since states that are rarely visited will have the same stepsize as states that are visited frequently, which potentially yields faster convergence for states that are frequently visited but slower convergence for states that are rarely visited. The condition that $\beta(t)$ is nonincreasing for large $t>T$ is not necessary for this case.
    \item State-dependent step size $\gamma_t(i) = \beta(n_t(i))$. Here, $n_t(i)$ is the number of times state $i$ was ever reached before time $t$ ($n_t = \sum_{\tau < t} \I_{i \in X_k^{\mu_\tau}}$), where $\I$ represents the indicator function. Thus, we only change the step size for state $i$ when state $i$ is visited. 
\end{itemize} 

Given either choice of step size, we will show that our algorithm converges:

\begin{theorem}\label{main_theorem}
If $J_t$ is defined as in \eqref{eq:update0} and $\gamma_t(i) = \beta(n_t(i))$ or $\gamma_t(i) = \beta(t)$, then $J_t $ converges almost surely to $J^*$.
\end{theorem}

It turns out that proving the convergence of the second type of step size is more challenging than the corresponding proof for the first type of step size. However, in practice, the second type of step size leads to much faster convergence and hence, it is important to study it. We observed in simulations that the first step size rule is infeasible for problems with a large number of states since the convergence rate is very slow. Therefore, in our simulations, we use the second type of step size rule to compare the advantages of updating the value function for each state visited along a trajectory over updating the value function for just the first state in the trajectory.

\cite{tsitsiklis2002convergence} considers a case where $p$ is nonuniform and the value for only the initial state $i = X_0^{\mu_t}$ is updated in each iteration. Our algorithm discards less information than that of \cite{tsitsiklis2002convergence}, but we require stronger assumptions on the MDP structure.

\section{Proof of the Main Result}





The key ideas behind our proof are the following. Once a state in a recurrent class is reached in an iteration, every state in that class will be visited with probability one in that iteration. Thus, if there is a non-zero probability of reaching every recurrent class, then each recurrent class is visited infinitely many times, and the results in \cite{tsitsiklis2002convergence} for the synchronous version of the OPI can be applied to each recurrent class to show the convergence of the values of the states in each such class. Next, since the rest of the graph is an acyclic graph, by a well-known property of such graphs, the nodes (states of the Markov chain) can be arranged in a hierarchy such that one can inductively show the convergence of the values of these nodes. At each iteration, we have to show that the conditions required for the convergence of stochastic approximation are satisfied. If the step-sizes are chosen to be state-independent, then they immediately satisfy the assumptions required for stochastic approximation. If the step-sizes are state-dependent, then a martingale argument shows that they satisfy the required conditions. We also verify that the noise sequence in the stochastic approximation algorithm satisfies the required conditions.

\subsection{Convergence for recurrent states}

Recall that our states can be decomposed as $S = \scriptT \sqcup \scriptR_1 \sqcup \scriptR_2\sqcup \ldots \sqcup \scriptR_m$, where the $\scriptR_j \forall j=1, \ldots, m$ are closed, irreducible recurrent classes under any policy. To show convergence of our algorithm, we will first show that the algorithm converges for each recurrent class $\scriptR_j$, then use this fact to show convergence for the transient states $\scriptT$. The proof will differ slightly for our two choices of the step size $\gamma_t(i)$, so we will consider each case separately.

\subsubsection{Step size $\gamma_t(i) = \beta(n_t(i))$}
Consider our iterative updates, restricted to the set of states $\scriptR_j$. Since $\scriptR_j$ is a closed, irreducible recurrent class, once any state in $\scriptR_j$ is visited, so will every other state. Recall the version of our state update without $v_t$ given by \eqref{eq:updatewithoutv} under policy $\mu_t$. Using our choice of $\gamma_t(i)$, the update has exactly the same step size for every state in $\scriptR_j$. We define $n_t(\scriptR_j)$ as the shared $n_t(i)$ for each state $i \in \scriptR_j$, and then for states $i \in \scriptR_j$, \eqref{eq:updatewithoutv} becomes:
\begin{equation*}
J_{t+1}(i) =  
	      \begin{cases}
               (1-\beta(n_t(\scriptR_j)))J_t(i) + \beta(n_t(\scriptR_j))(J^{\mu_t}(i) +  w_t(i))& \text{  if $N_t(i) < \infty$}\\
               J_t(i) & \text{ otherwise}\\
            \end{cases}
\end{equation*}

Now, consider only the steps $t_1, t_2, \ldots$ of the algorithm such that $\scriptR_j$ is visited by the trajectory $X_k^{\mu_t}$, so $n_{t_k}(\scriptR_j) = k-1$. Given our choice of step size, the above update becomes
\begin{equation*}
J_{t_{k+1}}(i) = (1-\beta(k-1))J_{t_k}(i) + \beta(k-1)(J^{\mu_{t_k}}(i) +  w_{t_k}(i)),
\end{equation*}
where the noise $w_{t_k}(i)$ only depends on the evolution of $X_k^{\mu_{t_k}}$ in the recurrent class $\scriptR_j$. This is identical to the algorithm considered by Tsitsiklis in \cite{tsitsiklis2002convergence}. Noting that $\sum_{k=1}^\infty \beta(k-1) = \infty$ and $\sum_{k=1}^\infty \beta^2(k-1) < \infty$ by our assumptions on $\beta$, by Proposition 1 from Tsitsiklis, we have that $J_t(i) \cas J^*(i)$ for all $i \in \scriptR_j$.

\subsubsection{Step size $\gamma_t(i) = \beta(t)$}

Again, consider our iterative updates restricted to $\scriptR_j$. We define $q_{\mu_t}(\scriptR_j)$ as the common probability of reaching any state in $\scriptR_j$. Then, we adapt the version of the update containing the noise term $v_t$ from \eqref{eq:updatewithv} into an update for each state in $\scriptR_j$ using our choice of $\gamma_t$:
\begin{equation*}
J_{t+1}(i) = \left(1-\beta(t) q_{\mu_t}(\scriptR_j)\right)J_t(i) + \beta(t) q_{\mu_t}(\scriptR_j)(J^{\mu_t}(i) +  w_t(i) + v_t(i))
\end{equation*}
The convergence of the above algorithm essentially follows from \cite{tsitsiklis2002convergence} with a minor modification. Since we have assumed that $q_{\mu_t}(\scriptR_j)$ is lower bounded, even though the step sizes are random here, the stochastic approximation results needed for the result in \cite{tsitsiklis2002convergence} continue to hold.

\subsection{Convergence for transient states} \label{Convergence for transient states}

Since the reachability graph $G$ restricted to transient states is a directed acyclic graph, it admits a reverse topological sort of its vertices $(x_0, x_1, x_2, \ldots, x_L)$, such that for each $i,j \leq L$, if $(x_i, x_j) \in E$ then $i > j$ (for reference, see \cite{topological}). We will inductively prove that $J_t(x_i) \cas J^*(x_i)$ for all $i \leq L$. 

We begin our induction with $x_0$. Since $x_0$ is transient, it must have at least one neighbor, and because it is first in the topological sort, its only neighbors $N(x_0)$ in $G$ are members of recurrent classes. From the previous section, we know that for all such neighbors $j$, $J_t(j) \cas J^*(j)$. Since these neighboring value functions converge to the optimal value, one can show that the greedy policy at state $x_0$ converges to an optimal policy. For convenience, we present this result as a lemma. A similar result is proved in Proposition 4.5 and Corollary 4.5.1 in \cite{bertsekas1978stochastic}.

\begin{lemma}\label{lem_policy_conv}
	For any state $x$, let $N(x)$ be the set of its neighbors in the reachability graph $G$. Suppose that for all $i \in N(x), J_t(i) \to J^*(i)$. Then, there exists a finite time T for which $\mu_t(x) = \mu^*(x)$ for all $t \geq T$.
\end{lemma}

Now, using Lemma \ref{lem_policy_conv}, let $T(i)$ be the minimum time after which $\mu_t(i) = \mu^*(i)$ for any optimal policy $\mu^*$. Now, let $A_n(i)$ be the event that $T(i) = n$ for $n \in \N \cup \{\infty\}$. Since $J_t(j)$ converges almost surely for all neighbors of $x_0$, $\P(A_\infty(x_0)) = 0$. We examine the probability that $J_t(x_0)$ does not converge to $J^*(x_0)$. The method is similar to the method in the errata of \cite{tsitsiklis2002convergence}.

\begin{align*}
    \P(J_t(x_0) \nrightarrow J^*(x_0)) &= \P(J_t(x_0) \nrightarrow J^*(x_0), A_\infty(x_0)) + \sum_{n=1}^\infty \P(J_t(x_0) \nrightarrow J^*(x_0) , A_n(x_0))\\
    &= \sum_{n=1}^\infty \P(J_t(x_0) \nrightarrow J^*(x_0) , A_n(x_0))
\end{align*}

We now analyze $\P(J_t(x_0) \nrightarrow J^*(x_0), A_n(x_0))$. For each integer $n \geq 0$, define a sequence $Z_t^{(n)}$ for $t \geq n$ such that $Z_n^{(n)} = J_n(x_0)$ and
\begin{equation}
    Z_{t+1}^{(n)} = (1-q_{\mu_t}(x_0)\gamma_t(x_0))Z_t^{(n)} + q_{\mu_t}(x_0)\gamma_t(x_0)(J^*(x_0) +  w_t(x_0) + v_t(x_0)). \label{eq:standard}
\end{equation}

$Z_t^{(n)}$ is now in a standard form for a stochastic approximation. We will use the following standard theorem adapted from Lemma 1 of \cite{singh2000convergence}  to prove convergence of \eqref{eq:standard} to $J^*(x_0)$:

\begin{lemma}\label{lem_stochastic_approx}
	Let $(x_t)_{t \in \N}, (w_t)_{t \in \N}, $ and $(\eta_t)_{t \in \N}$ be three sequences of scalar random variables such that $x_t$, $\eta_t$, and $w_{t-1}$ are $\scriptF_{t-1}$-measurable. Consider the update
	\begin{equation*}
		x_{t+1} = (1 - \eta_t)x_t + \eta_t w_t.
	\end{equation*}
	Assume the following conditions are met:
	\begin{enumerate}
		\item There exist finite constants $A, B$ such that $\E[w_t^2 | \scriptF_{t-1}] \leq A|x_t|^2 + B$ for all $t$.
		\item $\E[w_t | \scriptF_{t-1}] = 0$ for all $t$.
		\item $\eta_t \in [0, 1]$.
		\item $\sum_{t=0}^\infty \eta_t = \infty$ w.p. 1.
		\item $\sum_{t=0}^\infty \eta_t^2 < \infty$  w.p. 1.
	\end{enumerate}

Then, the sequence $x_t$ converges almost surely to $0$: $x_t \cas 0.$
\end{lemma}

To use Lemma \ref{lem_stochastic_approx}, we define our $\scriptF_t := \{(w_\tau)_{\tau \leq t}, (v_\tau)_{\tau \leq t}, (X_k^{\mu_\tau})_{\tau \leq t} \}$. It is straightforward to establish the following result, which we state without proof: 

\begin{lemma}\label{lem_noise}
$\E[w_t +v_t | \scriptF_{t-1}] = 0$ and $\E[\|w_t + v_t\|_\infty^2 | \scriptF_{t-1}] \leq D,$ for some constant $D$.
\end{lemma}

Finally, we need to demonstrate that for our step sizes $\gamma_t(i)=\beta(t)$ and $\gamma_t(i)=\beta(n_t(i))$, the effective step size $q_{\mu_t}(i)\gamma_t(i)$ almost surely satisfies 
\begin{equation}
\sum_{t=0}^\infty q_{\mu_t}(i)\gamma_t(i) = \infty, \qquad \sum_{t=0}^\infty q_{\mu_t}^2(i)\gamma_t^2(i) < \infty.\label{eq:effstep}
\end{equation}
Towards this, we introduce the following:

\begin{lemma}\label{lem_step_size}
For $\gamma_t(i) = \beta(n_t(i))$ and $\gamma_t(i) = \beta(t)$, \eqref{eq:effstep} holds almost surely for each state $i$.
\end{lemma}
\begin{proof} Since $0 < \delta < q_{\mu_t}(i) \leq 1$, it is sufficient to show that $\sum_{t=0}^\infty \gamma_t(i) = \infty$ and $\sum_{t=0}^\infty \gamma^2_t(i) < \infty$ for all $i \in S$ almost surely. This is true by definition for $\gamma_t(i) = \beta(t)$, so it remains to show this for $\gamma_t(i) =\beta(n_t(i))$. 

First we show that $\sum_{t=0}^\infty \beta(n_t(i))=\infty$ almost surely. Observe that $n_t(i) \leq t $ for all $t$ since $n_t(i)$ represents the number of trajectories in the first $t$ trajectories where state $i$ was visited. For sufficiently large $t$, $\beta(t)$ is nonincreasing, so $\beta(n_t(i)) \geq \beta(t)$. Furthermore, since $\sum_{t=0}^\infty \beta(t) =\infty,$ we have that $\sum_{t=0}^\infty \beta(n_t(i)) = \infty.$

We will apply the martingale convergence theorem to show that $\sum_{t=0}^\infty \gamma^2(t) < \infty$ almost surely. Define sequences $Q_t$ and $R_t$ as follows:
\begin{equation*}
    Q_t = \frac{\I_{i \in X_k^{\mu_t}} - q_{\mu_t}(i)}{t} \qquad\qquad R_t = \sum_{\tau = 1}^t Q_\tau
\end{equation*}
Clearly, $\E[Q_t | \scriptF_{t-1}] = 0$ and $|Q_t| \leq 1/t.$ Next, consider $\E[Q_tQ_u]$ for $t > u.$ We note that
$$\E[Q_tQ_u|\scriptF_{t-1}]=Q_u\E[Q_t|\scriptF_{t-1}]=0.$$ Thus, $\E[Q_t Q_u]=0$ and similarly by considering $u>t,$ $\E[Q_tQ_u]=0 \,\forall t\neq u.$ Therefore,
$$\sup_t \E[R^2_t] =\sup_t \sum_{\tau=1}^t \E[Q_\tau^2]\leq \sup_t \sum_{\tau=1}^t\frac{1}{t^2}\leq \frac{\pi^2}{6} < \infty.$$ Thus, $R_t$ is a martingale and satisfies the conditions of the martingale convergence theorem, and therefore $R_t$ converges almost surely to some well-defined random variable $R_\infty$, i.e., $P(R_\infty<\infty)=1.$ Since 
\begin{equation*}
\lim_{t \to \infty} \sum_{\tau=1}^t \frac{\I_{i \in X_k^{\mu_\tau}} - q_{\mu_\tau}(i)}{\tau}
\end{equation*}
is finite almost surely, by Kronecker's lemma, we have
\begin{align*}
    \lim_{t\to\infty} \frac{1}{t} \sum_{\tau=1}^t (\I_{i \in X_k^{\mu_\tau}} - q_{\mu_\tau}(i)) &= 0\\
    \lim_{t\to\infty} \left(\frac{n_t(i)}{t} - \frac{\sum_{\tau=1}^t q_{\mu_\tau}(i)}{t} \right) &= 0
\end{align*}
almost surely. Since $q_{\mu_t}(i) \geq \delta > 0$ for all $t\geq 0$ and $i \in S$, we almost surely have
\begin{align*}
    \limsup_{t\to\infty} \frac{t}{n_t(i)} \leq \frac{1}{\delta}.
\end{align*}
This implies that for sufficiently large $t$, $\lfloor t\delta/2 \rfloor \leq n_t(i)$. We have assumed that, for sufficiently large $t$, $\beta$ is nonincreasing, so $\beta(n_t(i))\leq \beta(\lfloor t\delta/2 \rfloor),$ which implies $\beta^2(n_t(i))\leq \beta^2(\lfloor t\delta/2 \rfloor).$ Finally, using $\sum_{t=0}^\infty \beta^2(t) < \infty,$ there is almost surely some $T_1 < \infty$ (which may depend on the sample path), such that
\begin{align*}
\sum_{t=T_1}^\infty \beta^2(n_t(i)) \leq \sum_{t=T_1}^\infty \beta^2(\lfloor \frac{t\delta}{2} \rfloor) \leq \sum_{t=T_1}^\infty (\frac{2}{\delta}+1) \beta^2(t) < \infty. 
\end{align*}
The second inequality in the previous line follows from the fact that the value of $\lfloor t\delta/2\rfloor$ changes only at $t=0,\lceil 2/\delta \rceil, \lceil 4/\delta \rceil, \ldots$. This implies that $\sum_{t=0}^\infty \beta^2(n_t(i)) < \infty$ almost surely.
\end{proof} 
Thus, the recurrence in \eqref{eq:standard} takes the form required by Lemma \ref{lem_stochastic_approx}, with step size $q_{\mu_t}(x_0)\gamma_t(x_0)$ and noise term $w_t(x_0) + v_t(x_0)$. Conditions 1 and 2 in Lemma \ref{lem_stochastic_approx} are satisfied by Lemma \ref{lem_noise}. Condition 3 is clearly satisfied, because $\gamma_t(x_0) \in [0, 1]$. Conditions 4 and 5 are satisfied due to Lemma \ref{lem_step_size}. Therefore, by Lemma \ref{lem_stochastic_approx}, $Z_t^{(n)} \cas J^*(x_0)$ for all positive integers $n$. Now, we are ready to complete the proof. Conditioned on $A_n(x_0)$, we have $J_t(x_0) = Z_t^{(n)}(x_0)$ for all $t \geq n$. Therefore:
\begin{align*}
    \P(J_t(x_0) \nrightarrow J^*(x_0)) &= \sum_{n=1}^\infty \P(J_t(x_0) \nrightarrow J^*(x_0), A_n(x_0))\\
    &= \sum_{n=1}^\infty \P(Z_t^{(n)} \nrightarrow J^*(x_0), A_n(x_0))\\
    &\leq \sum_{n=1}^\infty \P(Z_t^{(n)} \nrightarrow J^*(x_0))\\
    &= 0 && \text{(Lemma \ref{lem_stochastic_approx})}
\end{align*}
This completes the proof that $J_t(x_0) \cas J^*(x_0)$. We then only need to complete the induction. For any $0 < i \leq L$, suppose that $J_t(x_j) \cas J^*(x_j)$ for all $j < i$. We define $Z_t^{(n)}$ analogously to above, so $Z_n^{(n)} = J_n(x_i)$ and:
\begin{equation*}
Z_{t+1}^{(n)} = (1-q_{\mu_t}(x_i)\gamma_t(x_i))Z_t^{(n)} + q_{\mu_t}(x_i)\gamma_t(x_i)(J^*(x_i) +  w_t(x_i) + v_t(x_i))
\end{equation*}
By the inductive assumption and because of convergence for every recurrent class, the $J_t(j)$ for all $j \in N(x_i)$ converge almost surely. If we define $T(x_i)$ in the same way as with $x_0$, then with probability 1, $T(x_i)$ is finite. By the same reasoning as the base case, then $J_t(i) \cas J^*(i).$

\section{Numerical Experiments}
\begin{figure}[ht]
    \centering
   \subfloat[]{
       \includegraphics[width=.5\textwidth]{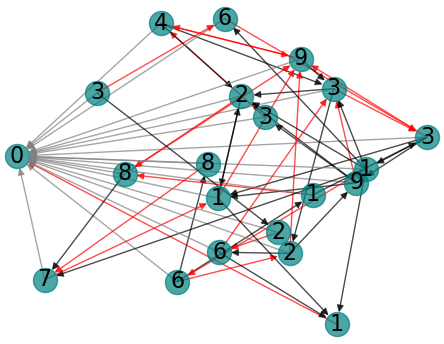}\label{fig:sub1a}
   }
   \subfloat[]{
        \includegraphics[width=.45\textwidth]{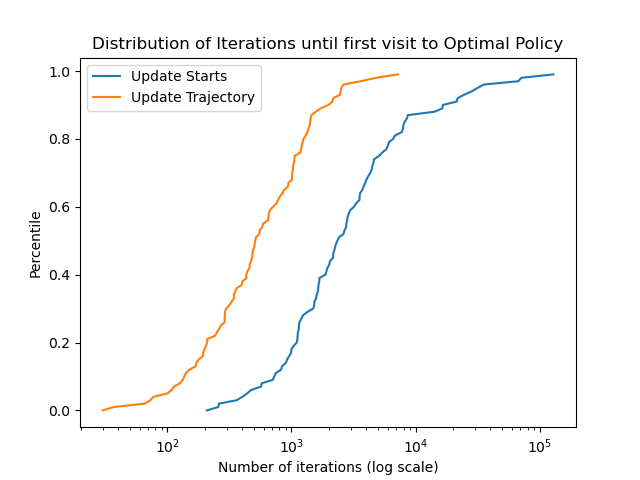}\label{fig:sub1b}
   }
   \label{fig:sim1}
   \caption{The MDP graph and results of our first experiment to compare the convergence speed of the asynchronous version of the algorithm in \cite{tsitsiklis2002convergence} and our variant presented in \ref{eq:updatewithv}, which updates every state along the sampled trajectory.}
\end{figure}

The primary difference between the algorithm we have analyzed and the variant previously analyzed in \cite{tsitsiklis2002convergence} is the update step. In \cite{tsitsiklis2002convergence}, only the value of a single, randomly-selected state is updated at each time step. However, we update every state visited by the trajectory sampled each time step. Because we update each visited state, we expect the variant we have analyzed to converge more quickly. In order to support this claim, we have performed two experiments which demonstrate faster convergence. Note that in the present section, we use \textit{rewards} instead of \textit{costs} where we seek to \textit{maximize} instead of \textit{minimize} cumulative discounted rewards with our policies. All of our results still hold when we use \textit{maximums} instead of \textit{minimums} to determine the policy that \textit{maximizes} the expected cumulative discounted reward.

In the first experiment, we have a Markov chain with a single absorbing state shown in Figure \subref{fig:sub1a}, where the absorbing state has label 0. All edges $(i,j)$ in the figure represent a possible transition from node $i$ to $j$. At each state $i \neq 0$, there is an action $j$ associated with edge $(i,j)$ out of state $i$, such that taking action $j$ transitions to state $j$ with probability $0.6$ and transitions to a different random neighbor of node $i$ chosen uniformly at random with probability $0.4$. If there is only edge out of state $i$, then the only action deterministically transitions along that edge. For all nonzero states in Figure \subref{fig:sub1a}, the label of the state corresponds to the reward of taking any action in that state (equivalently, the cost is the negation of the reward). The red arrows correspond to the optimal action in each state. This example is similar to taking $\epsilon$-greedy actions in an MDP with deterministic state transitions.

We implement both our algorithm given in (7) and the variant studied in \cite{tsitsiklis2002convergence} which only updates a single state each iteration, and compare the number of iterations required for convergence. The results over 100 trials, assuming a discount factor of $\alpha = 0.9$ and a step size of $1/n_t(i)$, can be found in Figure \subref{fig:sub1b}. The distribution of the starting state for each iteration was assumed to be uniformly random for both algorithms. Each algorithm was run until the first time $t$ that $\mu_t = \mu^*$, and we graphed the empirical distributions of the number of iterations required. On average, our algorithm (updating along the entire trajectory) required only about 854 iterations, compared to the algorithm from \cite{tsitsiklis2002convergence}, which required 7172 iterations on average when updating only the starting state of the trajectory each time step.

\begin{figure}[ht]
    \centering
    \subfloat[]{
        \includegraphics[width=.85\textwidth]{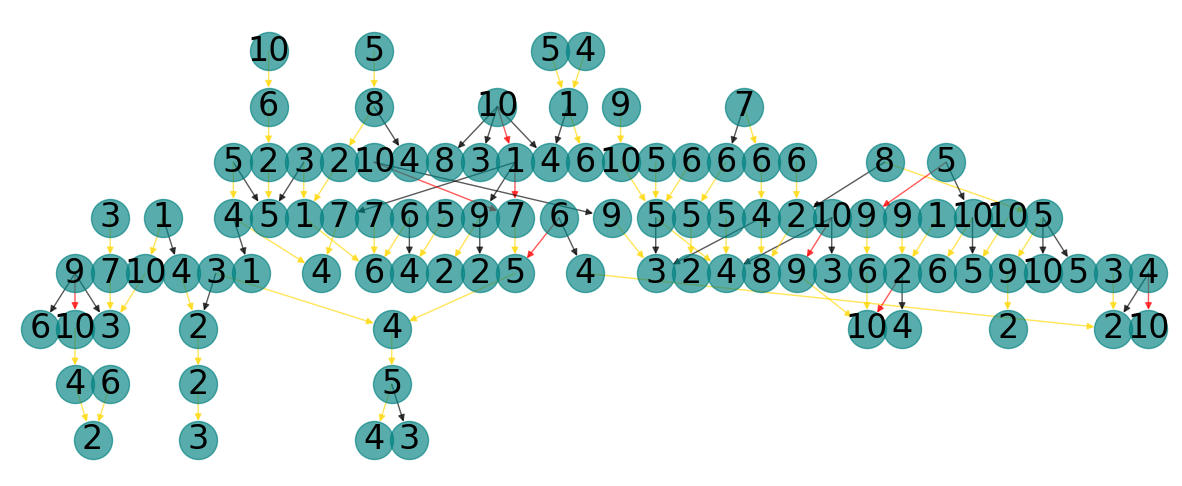}\label{fig:sub2a}
    }\\
    \subfloat[]{
        \includegraphics[width=.45\textwidth]{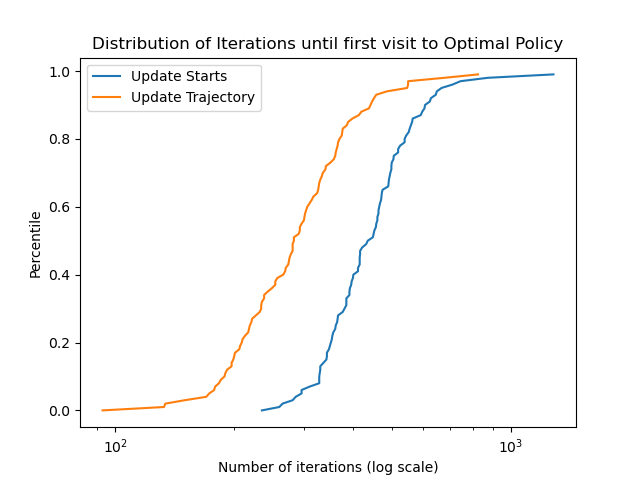}\label{fig:sub2b}
    }
    \label{fig:sim2}
   \caption{The MDP graph and results of our second experiment comparing the asynchronous algorithm from \cite{tsitsiklis2002convergence} with our variant}
\end{figure}

In the second example, we consider a different stochastic shortest path problem on the acyclic graph, shown in Figure \subref{fig:sub2a}. In this example, there are two actions, $j_1$ and $j_2$, associated with each edge $(i,j)$. If action $j_1$ is taken, then the reward in the label for node $i$ is accrued and a transition occurs as in the previous example, where the edge $(i,j)$ is taken with probability 0.6 and a different uniformly random edge is taken with probability $0.4$. The action $j_2$ allows for a more certain reward, at a cost; the probability of taking the chosen edge is increased to 0.8, but the reward is decreased by 1.

Again, we compare our algorithm to the variant studied in \cite{tsitsiklis2002convergence} for this problem. The optimal policy is given by the red and yellow arrows in Figure \subref{fig:sub2a}, where yellow arrows are associated with $j_1$ and red arrows with $j_2$. The distribution of iterations required for convergence can be found in Figure \subref{fig:sub2b}. Again, updating the entire trajectory (300 iterations on average) is more efficient than updating a single state (455 iterations on average).

\section{Extensions}

Thus far, we have presented a proof of convergence for a certain class of discounted MDPs with deterministic costs. However, the same ideas we have used can be easily extended to a number of related settings. In this section, we will discuss extensions to stochastic shortest path and game theoretic versions of the problem. We will also extend the results to a setting where we
assume knowledge of clusters of states with the same value function.

\subsection{Stochastic Shortest Path Problem} \label{SSP}
In a stochastic shortest path (SSP) problem, the goal is to minimize the cumulative cost over all policies. It is the undiscounted MDP problem, where the discount factor $\alpha$ is set to 1 and the cost-to-go $J^\mu(i)$ becomes
\begin{equation*}
J^{\mu}(i) = \mathbb{E}\left[\sum_{k=0}^\infty c(X_k^\mu, \mu(X_k^\mu)) | X_0^\mu = i\right].
\end{equation*}

To account for the lack of a discount factor, we will need to adjust our assumptions accordingly. We again assume that the state and action spaces are finite and we assume that Assumptions \ref{assumption_reachability} and \ref{assumption_common_transitions} hold as in the discounted case. However, instead of allowing the cost to infinitely accumulate in one of several recurrent classes, we require a different structural assumption, which combines all recurrent classes into one absorbing state and guarantees that the cost remains finite under every policy:
\begin{assumption}\label{assumption_ssp}
	There is a unique absorbing state 0, which incurs a cost of 0 under every action. For notational convenience, we will denote the state space for the SSP as $S \cup \{0\}$, with $S = \{1, \ldots, n\}$ as before. We assume the subgraph of the reachability graph induced by $S$ is acyclic.
\end{assumption}

We define our algorithm identically to the discounted case, but with $\alpha = 1$. 
The update proceeds using \eqref{eq:update0}. This procedure can be shown to converge, similarly to the discounted case:

\begin{theorem}
	Given Assumptions \ref{assumption_reachability}, \ref{assumption_common_transitions}, and \ref{assumption_ssp}, if $J_t$ is updated as in \eqref{eq:update0} and $\gamma_t(i) = \beta(n_t(i))$ or $\gamma_t(i) = \beta(t)$, then $J_t$ converges almost surely to $J^*$.
\end{theorem}
\begin{proof}
	The proof for this result follows the proof given in section 6.2, of the convergence for transient states in the discounted case. Due to our assumptions, the nonzero states of the SSP form an acyclic graph, so they admit a reverse topological sort $(x_1, x_2, \ldots, x_n)$, where in the reachability graph $G = (S, E)$, $(x_i, x_j) \in E$ implies $i > j$. Thus, state $x_1$ can only transition to the absorbing state 0, and for all time $t$, we have $J^{\mu_t}(x_1) = J^*(x_1)$. It is straightforward to show that Lemmas \ref{lem_noise} and \ref{lem_step_size} continue to hold for the SSP problem. Therefore, by a simple stochastic approximation argument, $J_t(x_1) \cas J^*$.
	
	The proof proceeds by induction in the same manner as in the undiscounted case. For any $k > 1$, assuming $J_t(x_i) \cas J^*(x_i)$ for all $i < k$, we examine $J_t(x_k)$. It is straightforward to show that Lemma \ref{lem_policy_conv} holds for the SSP problem. By an argument analogous to the one used above for $x_1$, then $J_t(x_k) \cas J^*(x_k)$.
\end{proof}

\subsection{Alternating Zero-Sum Game} 
We consider a finite-state stochastic shortest path game with two players: player 1 and player 2.  Player 1 seeks to minimize the cumulative cost, while player 2 works to maximize the cost. In general, player 1 and 2 can take simultaneous actions $u \in \scriptA_1(i)$ and $v \in \scriptA_2(i)$, respectively, in state $i$. Accordingly, transitions $P_{ij}(u,v)$ and costs $c(i,u,v)$ depend on both actions. These action spaces are often not finite, for example, to allow for mixed strategies for each player. Given a policy $\mu$ for player 1 and $\nu$ for player 2, we can define a cost function $J^{\mu, \nu}$:

\begin{equation*}
    J^{\mu, \nu}(i) = \E\left[\sum_{k=0}^\infty c(X_k, u_k, v_k) | X_0 = i, u_k = \mu(X_k), v_k = \nu(X_k)\right]
\end{equation*}

The goal in solving stochastic shortest path games is to find a Nash equilibrium solution $J^*$, such that
\begin{align*}
    \inf_{\mu} \sup_{\nu} J^{\mu, \nu}(i) = \sup_{\nu} \inf_{\mu} J^{\mu, \nu}(i) = J^*(i). 
\end{align*}

When the value of a game exists, it can be found as the solution to the minimax Bellman equation $TJ^* = J^*$, where $T$ is the minimax Bellman operator defined by
\begin{equation*}
    (TJ)(i) = \inf_u \sup_v \left[c(i, u, v) + \sum_{j}P_{ij}(u,v)J(j)\right]
\end{equation*}

If such a solution exists, then $J^*$ is the optimal value function for the game. One category of games where an equilibrium always exists is \emph{alternating games}, which we consider in this section (for more details, see section 2.3.3 of \cite{patekthesis}). In an alternating (also known as sequential) game, players take ``turns'' performing actions. The state space, outside of a single absorbing terminating state $0$, can be partitioned into two sets of states $S_1$ and $S_2$, where $S_1$ is the set of states where player 1 takes actions and $S_2$ is the set of states where player 2 acts. For states $i \in S_1$, the choice of action for player 2 is trivial and therefore $|\scriptA_2(i)| = 1$. Similarly, for states $i \in S_2$, $|\scriptA_1(i)| = 1$. Without loss of generality, we can combine states to assume $P_{ij}(u,v) = 0$ if $i$ and $j$ are either both in $S_1$ or both in $S_2$, so no player ever takes two turns in a row.

For the purposes of this section, we assume that the action spaces in each state are finite. In an alternating game, there is no need for mixed strategies, as at each step, the one-step minimax problem reduces to a simple minimum or maximum, depending on the current turn. Thus, we can combine the action pair $(u,v)$ into a single action and simplify the Bellman operator to a state-dependent min or max:
\begin{equation}
    (TJ)(i) = \begin{cases}
                   \min_u [c(i,u) + \sum_j P_{ij}(u)J(j)] & i \in S_1\\
                   \max_u [c(i,u) + \sum_j P_{ij}(u)J(j)] & i \in S_2.
    \end{cases}\label{eq:gameT}
\end{equation}
The following still holds:
$$TJ^*=J^*$$ for the operator $T$ in \eqref{eq:gameT}. Thus, we have the following:
\begin{equation}
    J^*(i) = \begin{cases}
                   \min_u [c(i,u) + \sum_j P_{ij}(u)J^*(j)] & i \in S_1\\
                   -\min_u [-c(i,u) - \sum_j P_{ij}(u)J^*(j)] & i \in S_2.\label{eq:negaJstar}
    \end{cases}
\end{equation} We define the following:
\begin{equation*}
    c'(i,u) = \begin{cases}
                   c(i,u) & i \in S_1\\
                   -c(i,u) & i \in S_2
    \end{cases}
\end{equation*}
and
\begin{equation*}
    J'(i) = \begin{cases}
                   J^*(i) & i \in S_1\\
                   -J^*(i) & i \in S_2.
    \end{cases}
\end{equation*}
Substituting $c'(i,u)$ and $J'(i)$ in equation \eqref{eq:negaJstar}, we arrive at the well-known \textit{negamin formulation} of the problem:
\begin{equation*}
    J'(i) = \min_u [c'(i,u) - \sum_j P_{ij}(u)J'(j)].
\end{equation*}
We denote the corresponding negamin Bellman operator as $T'$:
\begin{equation*}
    (T'J)(i) = \min_{u} \left[ c'(i,u) - \sum_j P_{ij}(u) J(j) \right]
\end{equation*}
 The negamin formulation transforms the problem from alternating minima and maxima into a problem with only minima. It is also often used in practical search algorithms for games.\footnote{Many existing formulations use the related negamax formulation instead, which transforms the problem into one only requiring a maximum instead of a minimum. We use the negamin formulation to be more consistent with our reinforcement learning formulation in terms of costs. For an example of reinforcement algorithms using a negamax formulation, see Alpha Go \cite{silver2016mastering}.} Intuitively, the negamin formulation redefines the cost $c'(i,u)$ as the ``cost from the perspective of the current player'', where the cost for one player is the negative of the cost for their opponent. Similarly, it defines a notion of value $J'$ as the value of the game from the perspective of the current player, using these new costs.
 
This negamin Bellman equation is equivalent to the SSP Bellman equation from the previous section, but with a ``discount factor'' of $-1$. Thus, to find the original negamin value $J^*$ of the game, satisfying $J^*=TJ^*$, we instead follow the algorithm \eqref{eq:update0} to find $J'$ but with $J^{\mu_t}+w_t$ defined in terms of the negamin Bellman operator $T'$, with $\alpha = -1$. Then, the value $J^*$ from the original formulation can be recovered from $J'$, the optimal negamin solution. Under the same assumptions as the previous section, we then converge almost surely to the optimal value of the game:

\begin{theorem}
	If $X$ is a stochastic shortest path game satisfying Assumptions \ref{assumption_reachability}, \ref{assumption_common_transitions}, and \ref{assumption_ssp}, if $J_t$ is updated as in \eqref{eq:update0}, and $\gamma_t(i) = \beta(n_t(i))$ or $\gamma_t(i) = \beta(t)$, where $\sum_{t=0}^\infty \beta(t) = \infty,$ and $\sum_{t=0}^\infty \beta^2(t) < \infty$, then $J_t$ converges almost surely to $J'$, from which we deduce $J^*$, the value function corresponding to the stochastic shortest path game.
\end{theorem}

The proof is identical to that of the stochastic shortest path problem, given Lemmas \ref{lem_policy_conv} and \ref{lem_stochastic_approx} hold for SSP games, which can be easily shown.

\subsection{Aggregation}


In some problems with large state spaces, we may have additional knowledge of the structure of the problem, which we can use to allow our algorithm to converge more quickly. One of the simplest structural assumptions we can make is to assume knowledge that several of the states share the same value function. Then, we should be able to ``aggregate'' our estimates of the value function for each of these clusters of states, reducing the size of the value function representation from $n$ to $k$, where $k$ is the number of clusters. In this way, aggregation acts as a very special case of linear function approximation, where we know apriori that $J^*(i) = \theta^T\phi(i)$ for some state-dependent vector $\phi(i)$ and some $\theta$. Proving the convergence of OPI for linear function approximation would be an interesting future extension.

For aggregation, we again assume a discounted MDP $X$ satisfying Assumptions \ref{assumption_common_transitions}-\ref{assumption_acyclic_transient}. We further assume that we have apriori a clustering of states into $k$ clusters where all the states in the same cluster have the same value function. We denote the $k$ clusters by $\mathcal{C}_1, \mathcal{C}_2, \ldots, \mathcal{C}_k,$ where $\cup_{i = 1}^k C_i = \mathcal{C}$. Then, formally, our assumption about the clusters is:
\begin{assumption}
    \label{assumption_clusters} For each cluster $\scriptC_c$ and each pair of states $i,j \in \scriptC_c$, $J^*(i) = J^*(j)$.
\end{assumption}

We define $J^*(\scriptC_c)=J^*(i)$ for all states $i \in \scriptC_c$. In order to show convergence, we need additional assumptions about the structure of the reachability graph. These assumptions are as follows: 
\begin{assumption}\label{assumption_aggregation_tree}
The Markov chain subgraph consisting of the transient states is acyclic. All states that are not transient are absorbing states. Further, we assume that all states in the same cluster have the same maximum distance to an absorbing state.
\end{assumption}
In other words, the states other than the absorbing states are our transient states and their Markov chain graph forms an acyclic graph. 

Because all clusters share the same optimal value, it is no longer necessary to visit every node in the cluster with positive probability to converge to the optimal value. Instead, all clusters must have positive probability of being visited under every policy. For each cluster $\scriptC_c \in \scriptC$, we define $q'_{\mu_t}(\scriptC_c )$ similarly to the quantity $q_{\mu_t}(i)$ for state $i$ in equation \eqref{q_mu_def}, as the probability of reaching cluster $\scriptC_c $ using policy $\mu_t$:
\begin{equation*}
    q'_{\mu_t}(\scriptC_c ) = P(X_k^{\mu_t} \in \scriptC_c  \text{ for some } k, 0 \leq k < \infty).
\end{equation*}
Then, we can relax Assumption \ref{assumption_reachability} to the following weaker assumption:
\begin{assumption}\label{assumption_reachability2}
    $q'_\mu(\scriptC_c ) > 0 \;\forall \mu, \scriptC_c  \in \mathcal{C}.$
\end{assumption}

We can modify our main algorithm in \eqref{eq:update0} to obtain a new algorithm that uses knowledge of the clusters to potentially yield faster convergence and lower storage complexity. Under the aforementioned assumptions, we will prove convergence of our modified algorithm.

At each time step $t$ we maintain a $k$-dimensional vector $\theta_t \in \mathbb{R}^k$, where each component $\theta_t(\scriptC_c )$ is the current estimate of the value function of states in cluster $\scriptC_c $. For all states $i \in \scriptC_c$ where $\scriptC_c  \in \mathcal{C},$ $\phi(i)=\mathbf{1}_c$, where $\mathbf{1}_c \in \mathbb{R}^k$ represents the vector of zeros with a $1$ in position $c$. Thus, for state $i \in \scriptC_c$, the current value function estimate $J_t(i)$ takes the form 
\begin{equation}
J_t(i) = \theta_t^T\phi(i)=\theta_t(c). \label{Jt_up_agg}
\end{equation}
At the start of each time step $t$, we pick an initial state using a non-uniform probability distribution $p$ (noting that assumption \ref{assumption_reachability2} holds). We calculate the greedy policy $\mu_t$ with respect to $J_t$ and simulate a trajectory $X^{\mu_t}_k$ following $\mu_t$, observing the costs incurred. The first time each state $i\in \scriptC_c $ visited, we calculate the tail costs from state $i$ and call it $\tilde{J}^{\mu_t}(c)$. Note that by our assumptions, each cluster is visited at most once in each time step. We then use $\tilde{J}^{\mu_t}(c)$ to update the estimate of the value function for cluster $\scriptC_c $,  $\theta_c$, using a cluster-dependent step size $\gamma_t(c)$:
\begin{equation}
\theta_{t+1}(c) =  
\begin{cases}
(1-\gamma_t(c))\theta_t(c) + \gamma_t(c)\tilde{J}^{\mu_t}(c) & \text{ if $i \in X_k^{\mu_t} $ for some $i \in \scriptC_c$}\\
\theta_t(c) & \text{ if $i \notin X_k^{\mu_t}$ for all $i \in \scriptC_c$}.
\end{cases} \label{eq:updateagg}
\end{equation}

We are now ready to state the convergence result, which establishes convergence to $J^*(i) \forall i \in \mathcal{C}_c$:
\begin{theorem}
	Suppose that assumptions \ref{assumption_common_transitions}, \ref{assumption_acyclic_transient}, \ref{assumption_clusters}, \ref{assumption_aggregation_tree}, and \ref{assumption_reachability2} hold. Then, the algorithm described in \eqref{eq:updateagg} with $\gamma_t(c) = \beta(n_t(c))$ or $\gamma_t(c) = \beta(t)$, where $\sum_{t=0}^\infty \beta(t) = \infty,$ and $\sum_{t=0}^\infty \beta^2(t) < \infty$,  converges almost surely to $J^*$. Here, $n_t(c)$ represents the number of times \textit{cluster} $\scriptC_c$ was ever reached by time $t$.
\end{theorem}
\begin{proof}
The proof is similar to the proof without state aggregation and proceeds by induction. Before performing the induction, however, it is helpful to rewrite the update \eqref{eq:updateagg} in the form of \eqref{eq:updatewithv}, with noise terms incorporating the randomness of the sampled trajectory in random variables $w_t$ and $v_t$. First, we define the zero-mean noise term $w_t(c)$ that incorporates noise from the trajectory for all clusters $\scriptC_c \in \scriptC$. Thus, our update becomes:
\begin{equation*}
	\theta_{t+1}(c) = \begin{cases}(1 - \gamma_t(c))\theta_t(c) + \gamma_t(c)(J^{\mu_t}(i) + w_t(i)) & \text{ if $i \in X_k^{\mu_t} $ for some $i \in \mathcal{C}_c$}\\
	\theta_t(c) & \text{ if $i \notin X_k^{\mu_t}$ for all $i \in \mathcal{C}_c$},
	\end{cases}
\end{equation*}
Note that if state $i$ is never reached by the trajectory, $w_t(i)$ is defined to be 0. Now, we can further define $v_t(i)$ to capture the randomness present in the event $i \in X_k^{\mu_t}$ and rewrite the above update as:
\begin{equation*}
	\theta_{t+1}(c) = \left(1 - \gamma_t(c) \sum_{i \in \mathcal{C}_c} q_{\mu_t}(i)\right) \theta_t(c) + \gamma_t(c) \sum_{i \in \mathcal{C}_c} q_{\mu_t}(i) \left( J^{\mu_t}(i) + w_t(i) + v_t(i) \right),
\end{equation*}
where
\begin{equation*}
	v_t(i) := \left( \frac{\I_{i \in X_k^{\mu_t}}}{q_{\mu_t}(i)} - 1 \right)(-\theta_t(c) + J^{\mu_t}(i) + w_t(i)).
\end{equation*}
The key to the proof is the observation that when $J^{\mu_t}(i)$ is the same for every node $i \in \mathcal{C}_c$, i.e. $J^{\mu_t}(i)=J^{\mu_t}(j) \forall i, j \in \mathcal{C}_c,$ the above update becomes:
\begin{equation}
	\theta_{t+1}(c) = (1 - q'_{\mu_t}(c)\gamma_t(\ell))\theta_t(c) + q'_{\mu_t}(c)\gamma_t(c)\left(J^{\mu_t}(c) + w_t(c) + v_t(c)\right), \label{eq:updateaggv}
\end{equation}
where $q'_{\mu_t}(c)=\sum_{i \in \mathcal{C}_c} q_{\mu_t}(i)$, $w_t(c) = \sum_{i \in \mathcal{C}_c}\frac{q_{\mu_t}(i)}{\sum_{i \in \mathcal{C}_c} q_{\mu_t}(i)} w_t(i)$, and $v_t(c) = \sum_{i \in \mathcal{C}_c}\frac{q_{\mu_t}(i)}{\sum_{i \in \mathcal{C}_c} q_{\mu_t}(i)} v_t(i)$. This is in the standard stochastic approximation form, as $w_t$ and $v_t$ are zero-mean.
Furthermore, when $\mu_t(i)=\mu^*(i)$ for all states $i \in \mathcal{C}_c,$ we have that $J^{\mu_t}(i)=J^{\mu_t}(j)=J^*(\scriptC_c)$ for all $i,j\in \mathcal{C}_c$ and our update becomes:
\begin{equation}
	\theta_{t+1}(c) = (1 - q'_{\mu_t}(c)\gamma_t(\ell))\theta_t(c) + q'_{\mu_t}(c)\gamma_t(c)\left(J^*(\scriptC_c) + w_t(c) + v_t(c)\right), \label{eq:updateaggv_opt}
\end{equation}
which implies that $\theta_t(c)\to J^*(\scriptC_c)$ for cluster $\scriptC_c$. We can now proceed with the proof by induction. In the induction, we consider ``layers'' of the Markov chain graph, where a layer consists of a set of states with the same maximum distance to the abosrbing states. Our inductive hypothesis is that $J_t(i) \cas J^*(\scriptC_{c_1})$ for $i \in \scriptC_{c_1}$, where $\scriptC_{c_1}$ is any cluster in a given layer of the Markov chain graph. We then show via induction that $J_t(i) \cas J^*(\scriptC_{c_2})$ for $i \in \scriptC_{c_2}$, where $\scriptC_{c_2}$ is a cluster in the layer above the layer containing $\scriptC_{c_1}$. First, we show convergence for the clusters containing absorbing states. Note that for the absorbing states, $\mu_t(i)=\mu^*(i).$

Now, consider a cluster containing absorbing states, $\scriptC_a$. We get that $J^{\mu_t}(i) = J^*(\scriptC_a) \forall i \in \mathcal{C}_a$, and we can use the update in \eqref{eq:updateaggv_opt} to determine convergence of $\theta_t(a)$ which corresponds to cluster $\scriptC_a$ as follows: 
\begin{equation*}
\theta_{t+1}(a) = (1 - q'_{\mu_t}(a)\gamma_t(a))\theta_t(a) + q'_{\mu_t}(a)\gamma_t(a)(J^*(\scriptC_{\scriptA}) + w_t(a) + v_t(a)).
\end{equation*} 
Using Lemma~\ref{lem_stochastic_approx}, we can then easily show that $\theta_t(a) \cas J^*(\scriptC_{\scriptA}),$ which implies from \ref{Jt_up_agg} that $J_t(i) \cas J^*(\scriptC_{\scriptA}) = J^*(i)$ for all $i \in \scriptC_{\scriptA}$.

Now, for the induction, consider a layer $\ell$ and assume that $\theta_t(c) \cas J^*(\scriptC_c)$ for all clusters $\scriptC_c$ in layers ``below'' layer $\ell$. Consider a cluster in layer $\ell$, say, $\scriptC_{\ell}$. From \ref{Jt_up_agg}, we can see that $J_t(i)\to J^*(\scriptC_c)$ for all clusters $\scriptC_c$ in layers below $\ell.$ Then, by Lemma~\ref{lem_policy_conv}, there is some minimum finite time $T(\scriptC_{\ell})$ for which $\mu_t(i) = \mu^*(i)$ for all $i \in \scriptC_{\ell}$, $t \geq T(\scriptC_{\ell})$. Let $A_n(\scriptC_{\ell})$ be the event that $T(\scriptC_{\ell}) = n$ for $n \in \N$. We proceed analogously to the procedure in Section~\ref{Convergence for transient states}.

We define a sequence $Z_t^{(n)}$ for every integer $n \geq 0$, such that $Z_n^{(n)} = \theta_n(\ell)$ and
\begin{equation}
    Z_{t+1}^{(n)} = (1 - q'_{\mu_t}(\ell)\gamma_t(\ell))Z_t^{(n)} + q'_{\mu_t}(\ell)\gamma_t(\ell)\left(J^*(C_{\ell}) + w_t(\ell) + v_t(\ell)\right),
\end{equation}
for $t \geq n$. Notice that conditioned on the event $A_n(\scriptC_\ell)$, for all $t \geq n$, $Z_t^{(n)}=\theta_t(\ell)$. Using Lemma~\ref{lem_stochastic_approx}, we can show that $Z_t^{(n)} \cas J^*(C_{\ell})$ for all $n \geq 0$. Therefore, as in Section~\ref{Convergence for transient states}, we have
\begin{align*}
    \P(\theta_t(\ell) \nrightarrow J^*(C_{\ell})) &= \sum_{n=1}^\infty \P(\theta_t(\ell) \nrightarrow J^*(C_{\ell})), A_n(C_{\ell}))\\
    &= \sum_{n=1}^\infty \P(Z_t^{(n)} \nrightarrow J^*(C_{\ell}), A_n(C_{\ell}))\\
    &\leq \sum_{n=1}^\infty \P(Z_t^{(n)} \nrightarrow J^*(C_{\ell}))\\
    &= 0.
\end{align*}

Thus, $\theta_t(\ell) \cas J^*(C_{\ell})$ which means that the induction holds and that $\theta_t(i)\cas J^*(i) \forall i \in S.$
\end{proof}

\section{Conclusions}
In this paper, we presented a proof of convergence for an extension of an optimistic policy iteration algorithm presented in \cite{tsitsiklis2002convergence} in which the value functions of multiple states (visited according to a greedy policy) are updated in each iteration. We present simulation results which show that such an update scheme can speed up the convergence of the algorithm. 
We extended the results to the following cases, (i) stochastic shortest-path problems, (ii) zero-sum games, and (iii) aggregation.
To prove our result, we assumed that each stationary policy induces the same graph for its underlying Markov chain and the Markov chain graph has the following property: if each recurrent class is replaced by a single node, then the resulting graph is acyclic. An interesting future direction for research is to relax these assumptions.

\section{Acknowledgements}

The research presented here was supported in part by a grant from Sandia National Labs\footnote{Sandia National Laboratories is a multimission laboratory managed and operated by National Technology \& Engineering Solutions of Sandia, LLC, a wholly owned subsidiary of Honeywell International Inc., for the U.S. Department of Energy’s National Nuclear Security Administration under contract DE-NA0003525.  This paper describes objective technical results and analysis. Any subjective views or opinions that might be expressed in the paper do not necessarily represent the views of the U.S. Department of Energy or the United States Government.} and the NSF Grant CCF 1934986.

\newpage
\bibliographystyle{unsrt}
\bibliography{refs}

\end{document}